\documentclass{article}
\pdfoutput=1

\usepackage{microtype}
\usepackage{graphicx}
\usepackage{subfigure}
\usepackage{booktabs} 

\usepackage{hyperref}

\usepackage[accepted]{icml2023}

\usepackage{amsmath}
\usepackage{amssymb}
\usepackage{mathtools}
\usepackage{amsthm}

\usepackage[capitalize,noabbrev]{cleveref}

\renewcommand{\algorithmicrequire}{\textbf{Param:}}
\renewcommand{\algorithmicensure}{\textbf{Init:}}

\renewcommand{\algorithmiccomment}[1]{$//$ \; #1}

\newcommand{\Tau}{\mathrm{T}}

\newcommand{\M}{\mathcal{M}}
\newcommand{\B}{\mathcal{B}}

\newcommand{\A}{\mathcal{A}}
\newcommand{\SState}{\mathcal{S}}

\newcommand{\Prob}[1]{\mathbb{P}[#1]}

\newcommand{\R}{\mathbb{R}}

\usepackage{color}

\newcommand*\xbar[1]{%
   \hbox{%
     \vbox{%
       \hrule height 0.5pt 
       \kern0.5ex
       \hbox{%
         \kern-0.1em
         \ensuremath{#1}%
         \kern-0.1em
       }%
     }%
   }%
} 
\graphicspath{{figures/}}

\theoremstyle{plain}
\newtheorem{thm}{Theorem}[section]

\newtheorem{lem}[thm]{Lemma}

\newtheorem{cor}[thm]{Corollary}

\newtheorem{rem}[thm]{Remark}

\theoremstyle{definition}
\newtheorem{defn}{Definition}[section]

\theoremstyle{remark}

\newcommand*{\starnr}{\stepcounter{equation}\tag{\theequation}}


\newlength{\subcolumnwidth}

\newcommand{\nextsubcolumn}[1][]{%
  \cr\noalign{\hfill}
  \if\relax\detokenize{#1}\relax\else\hsize=#1\setlength{\subcolumnwidth}{\hsize}\fi
}

\theoremstyle{plain}

\theoremstyle{definition}

\theoremstyle{remark}

\usepackage[textsize=tiny]{todonotes}

\icmltitlerunning{Eventual Discounting Temporal Logic Counterfactual Experience Replay}

\begin{document}

\twocolumn[
\icmltitle{Eventual Discounting Temporal Logic Counterfactual Experience Replay}



\icmlsetsymbol{equal}{*}

\begin{icmlauthorlist}
\icmlauthor{Cameron Voloshin}{caltech}
\icmlauthor{Abhinav Verma}{penn}
\icmlauthor{Yisong Yue}{caltech,argo}
\end{icmlauthorlist}

\icmlaffiliation{caltech}{Caltech}
\icmlaffiliation{argo}{Latitude AI}
\icmlaffiliation{penn}{Penn State}

\icmlcorrespondingauthor{Cameron Voloshin}{cvoloshin@caltech.edu}

\icmlkeywords{Machine Learning, ICML}

\vskip 0.3in
]



\printAffiliationsAndNotice{}  

\begin{abstract}



Linear temporal logic (LTL) offers a simplified way of specifying tasks for policy optimization that may otherwise be difficult to describe with scalar reward functions.  However, the standard RL framework can be too myopic to find maximally LTL satisfying policies.  This paper makes two contributions.  First, we develop a new value-function based proxy, using a technique we call eventual discounting, under which one can find policies that satisfy the LTL specification with highest achievable probability.  Second, we develop a new experience replay method for generating off-policy data from on-policy rollouts via counterfactual reasoning on different ways of satisfying the LTL specification. Our experiments, conducted in both discrete and continuous state-action spaces, confirm the effectiveness of our counterfactual experience replay approach.

\end{abstract}

\section{Introduction}\label{sec:intro}

In the standard reinforcement learning (RL) framework, the goal is to develop a strategy that maximizes a reward function in an unknown environment. In many applications of RL, a practitioner is responsible for generating the reward function so that the agent will behave desirably after the learning process. However, it can be challenging to convey real-world task specifications through scalar rewards \cite{Randlv1998LearningTD, toromanoff2019deep, ibarz2018reward, zhang2021importance, ng1999policy}. Colloquially known as reward-shaping, practitioners often resort to using heuristic "breadcrumbs" \cite{Sorg2011} to guide the agent towards intended behaviors. Despite the ``reward function is enough'' hypothesis \cite{Sutton1998, silver2021reward}), some tasks may not be reducible to scalar rewards \cite{Abel2021}. 



In response to these challenges, alternative RL frameworks using Linear Temporal Logic (LTL) to specify agent behavior have been studied (see Section \ref{sec:related-work}). LTL can express desired characteristics of future paths of a system \cite{modelchecking}, allowing for precise and flexible task/behavior specification. For example, we may ask for a system to repeatedly accomplish a set of goals in specific succession (see Section \ref{sec:background} for more examples).

Without significant assumptions, there is no precise signal on the probability of a policy satisfying an LTL objective. Existing work overwhelmingly uses Q-learning with a sparse RL heuristic \citep{Bozkurt2020Q, cai2021modular} meant to motivate an agent to generate trajectories that appear to satisfy the task. First, these heuristics involve complicated technical assumptions obscuring access to non-asymptotic guarantees, even in finite state-action spaces. Second, sparse reward functions pose substantial challenges to any gradient-based RL algorithm since they provide poor signal to adequately solve credit assignment. Resolving sparsity involves using a hierarchical approach \cite{Bozkurt2020Q} or re-introducing reward-shaping \cite{Hasanbeig2018lcrl}. See Section \ref{sec:related-work} for further elaboration on prior work.


\textbf{Our contributions.} In this paper, we focus on model-free policy learning of an LTL specified objective from online interaction with the environment. 
We make two technical contributions.  First, we reformulate the RL problem with a modified value-function proxy using a technique we call \emph{eventual discounting}.  The key idea is to account for the fact that optimally satisfying the LTL specification may not depend on the length of time it takes to satisfy it (e.g., ``eventually always reach the goal'').  We prove in Section \ref{sec:discounting} that the optimal policy under eventual discounting maximizes the probability of satisfying the LTL specification.

Second, we develop an experience replay method to address the reward sparsity issue.  Namely, any LTL formula can be converted to a fully known specialized finite state automaton from which we can generate multiple counterfactual trajectories from a single on-policy trajectory.  We call this method \textit{LTL-guided counterfactual experience replay}.   We empirically validate the performance gains of our counterfactual experience replay approach using both finite state/action spaces as well as continuous state/action spaces using both Q-learning and Policy Gradient approaches.  



\section{Preliminaries}\label{sec:background}

We give the necessary background and examples to understand our problem statement and solution approach.
An \emph{atomic proposition} is a variable that takes on a truth value. An \emph{alphabet} over a set of atomic propositions $\text{AP}$ is given by 
$\Sigma = 2^{\text{AP}}$. For example, if $\text{AP} = \{x, y\}$ then $\Sigma = \{\{\}, \{x\}, \{y\}, \{x, y\}\}$. $\Delta(A)$ represents the set of probability distributions over a set $A$. 

\subsection{Running Example}

\begin{figure*}[!ht]
\minipage{.2\textwidth}
  \includegraphics[width=\linewidth]{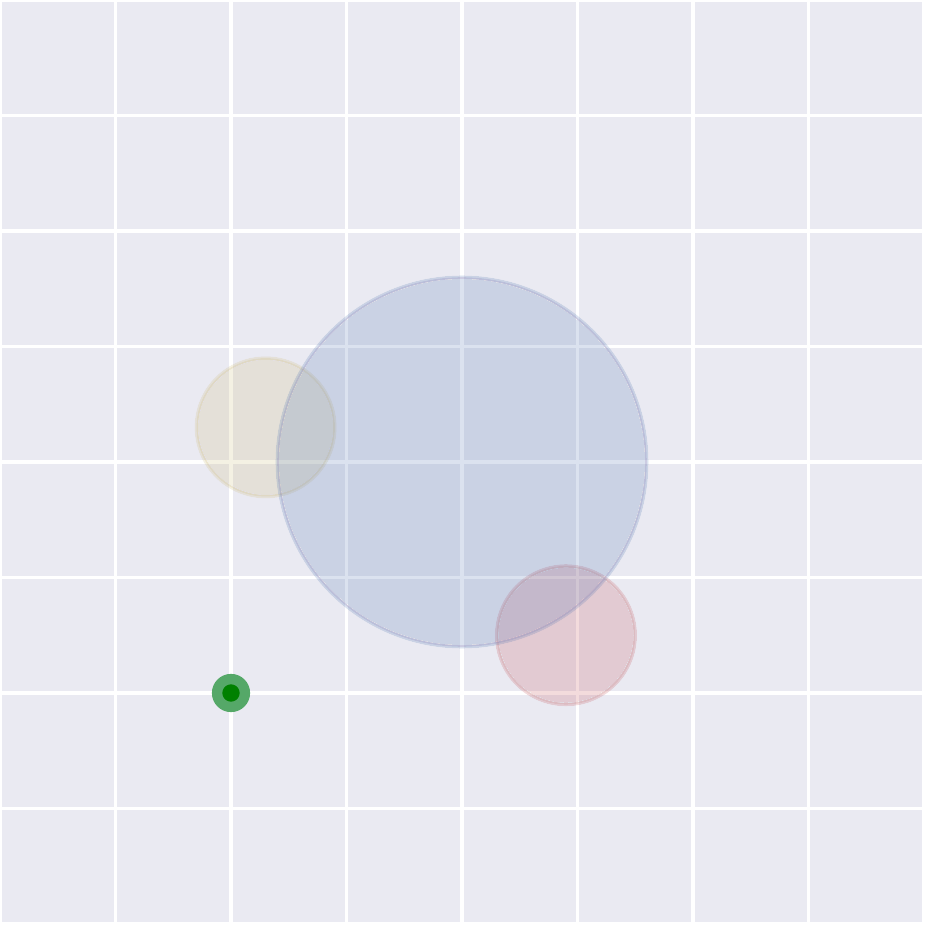}
\endminipage\hfill
\minipage{.25\textwidth}
  \includegraphics[width=\linewidth]{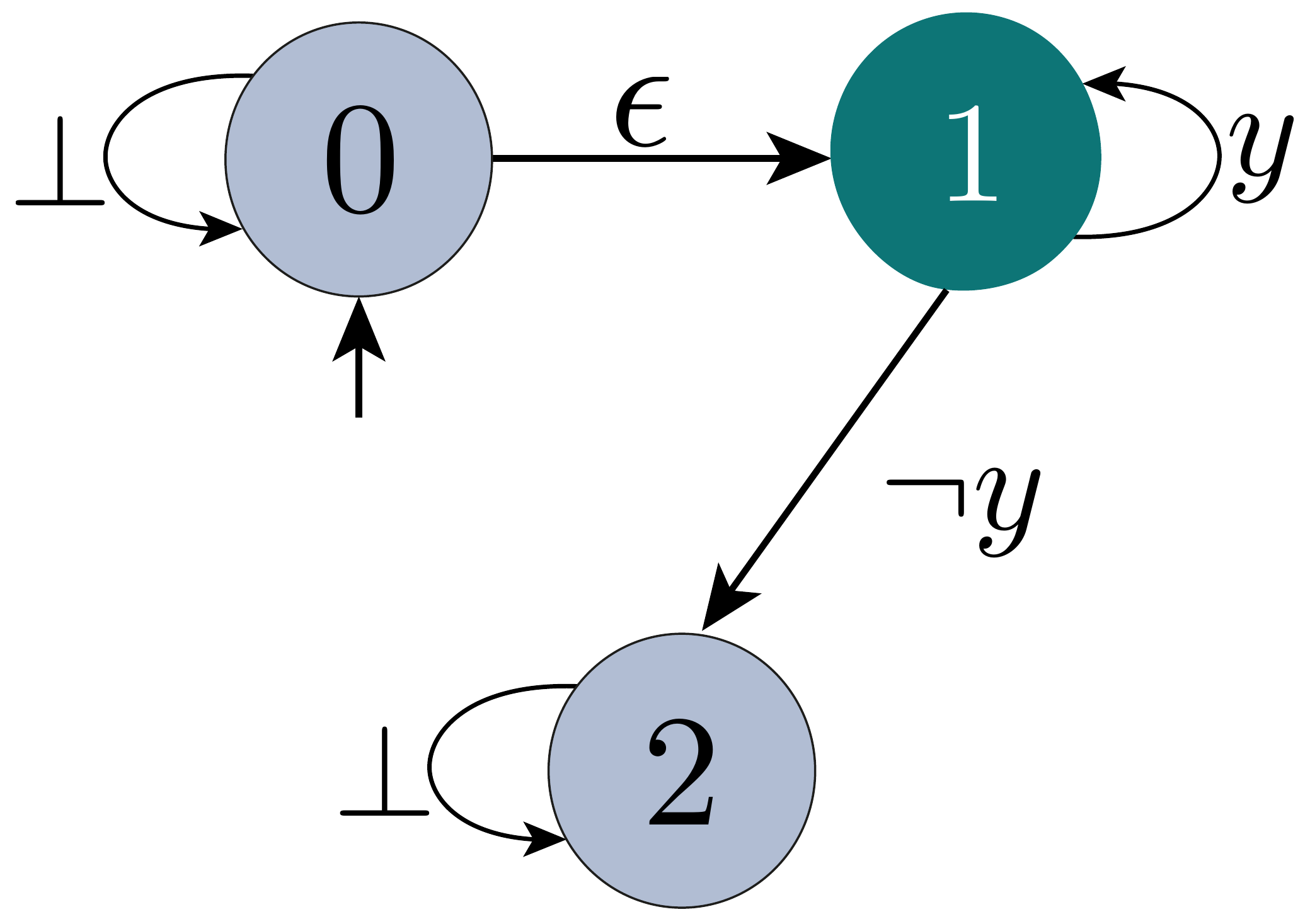}
\endminipage\hfill
\minipage{.3\textwidth}
  \includegraphics[width=\linewidth]{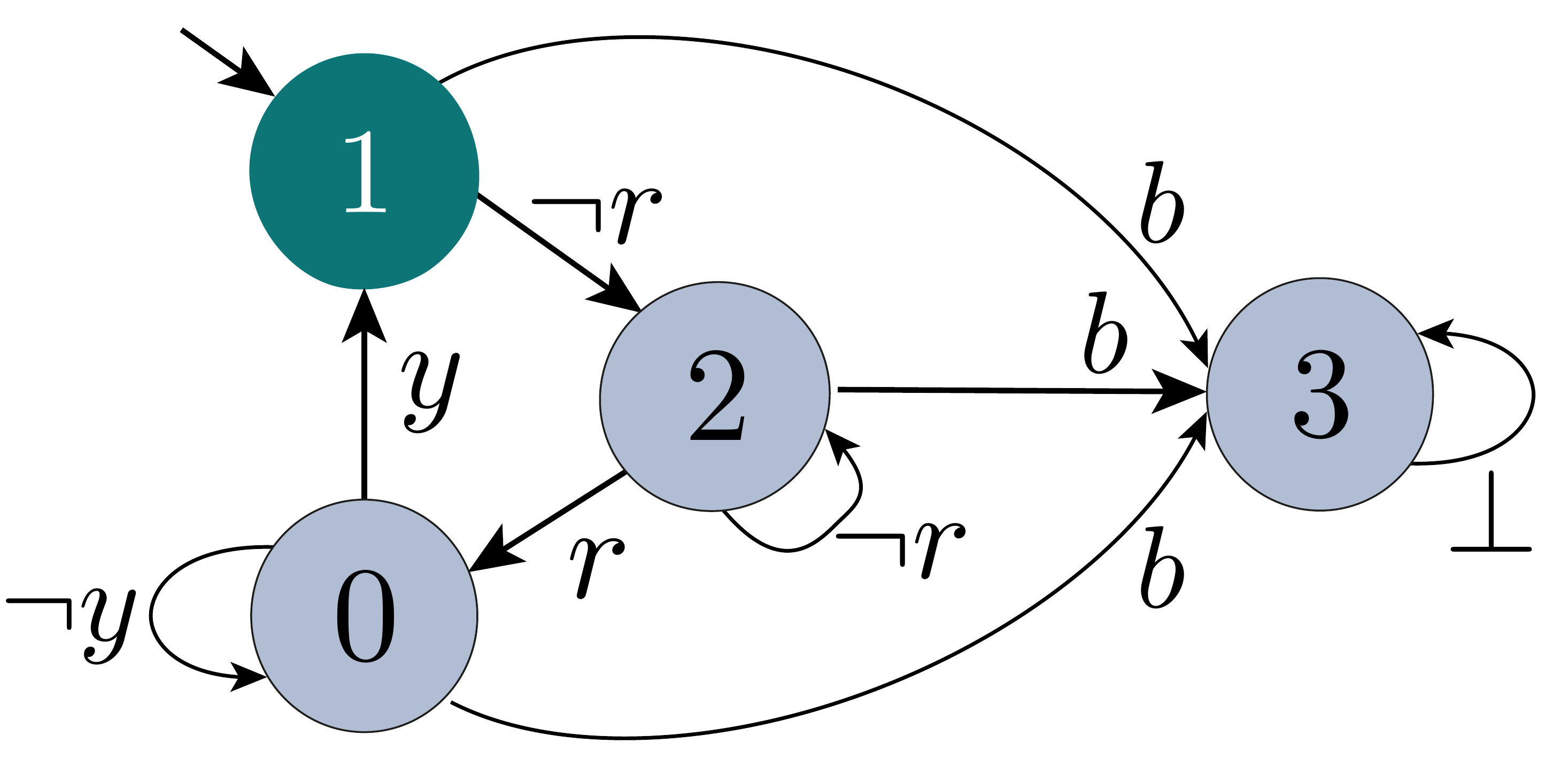}
\endminipage\hfill
\minipage{.25\textwidth}
  \includegraphics[width=\linewidth]{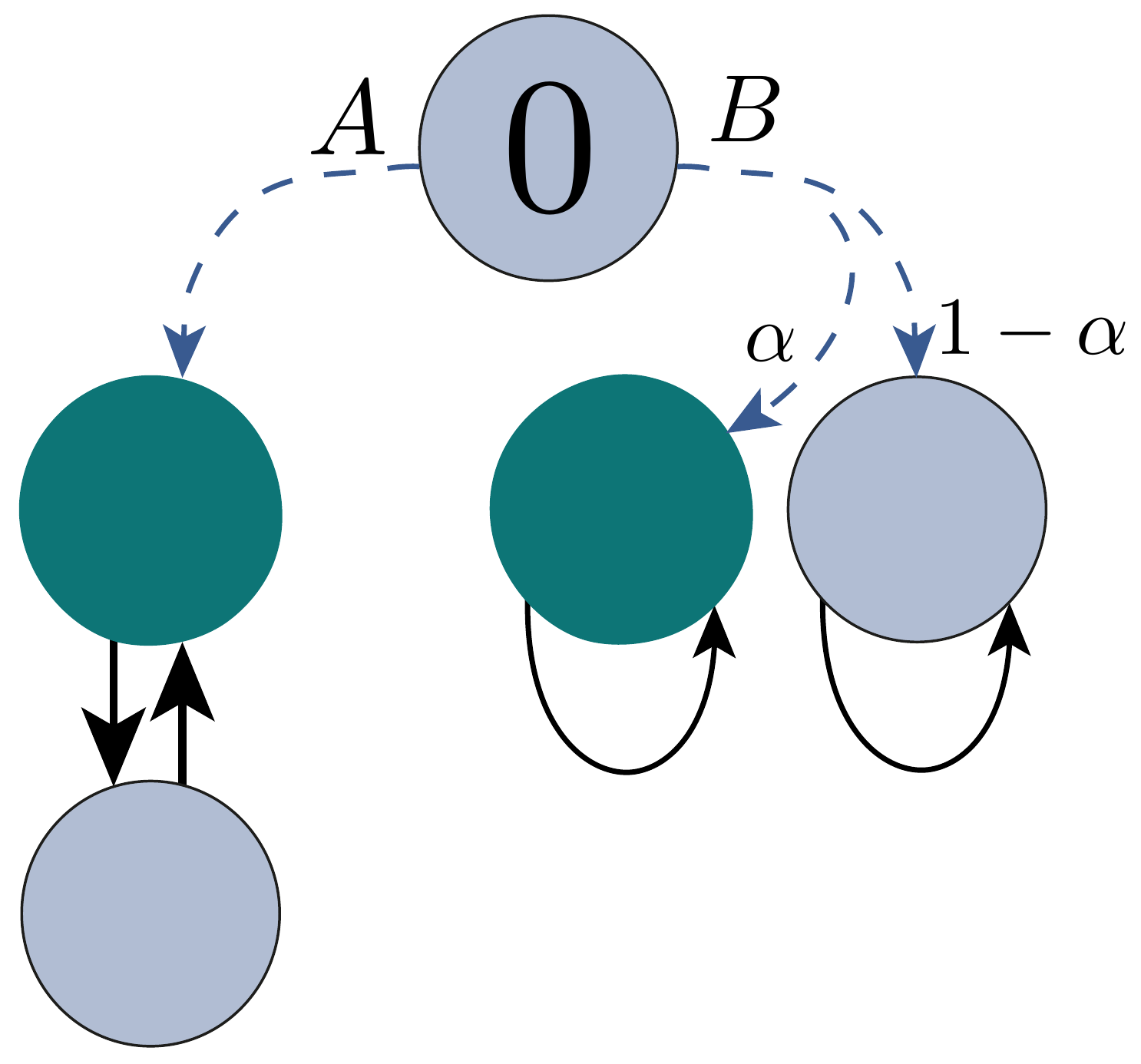}
\endminipage\hfill
\caption{\textit{Examples.} \textit{\textbf{First:}} Illustration of the Flatworld environment. The agent is a green dot and there are 3 zones: yellow, blue and red. \textit{\textbf{Second:}} LDBA $\B$ for ``$FGy$''. $\SState^{\B^\ast} = \{1\}$, denoted by green circle. The initial state is $b_{-1} = 0$. \textit{\textbf{Third:}} LDBA $\B$ for ``$GF(y \;\&\; X F r) \;\&\; G\lnot b$''. $\SState^{\B^\ast} = \{1\}$, denoted by green circle. The initial state is $b_{-1} = 1$. \textit{\textbf{Fourth:}} For this example, an agent starting in state $0$ and solving $\arg\max_{\pi \in \Pi} \mathbb{E}_{\tau \sim \Tau_{\pi}^P}[\sum_{i=0}^\infty \gamma^i \mathbf{1}_{\{b_i \in \SState^{\B^\ast}\}}]$ where $\SState^{\B^\ast}$ is illustrated as the green circles would choose to take action $B$ with probability $1$ if $\alpha \in (1/2,1)$ for any $\gamma \in [0,1]$. Such a policy has $\Prob{\pi \models \varphi} = \alpha < 1$. However the probability optimal policy deterministically takes action $A$, with $\Prob{\pi^\ast \models \varphi} = 1$. This illustrates catastrophic myopic behavior.}
 \label{fig:motivation}
 \vspace{-.1in}
\end{figure*}

We will be using the environment illustrated in Figure \ref{fig:motivation} (First) as a running example. The agent is given by a green dot and there are 3 circular regions in the environment colored yellow, blue and red. The AP are given by $\{y, b, r\}$, referring to the respective colored zones. Elements of $\Sigma$ indicate which zone(s) the agent is in. 

\subsection{MDPs with Labelled State Spaces} 
\label{sec:mdp}

Following a similar notation as \citet{VoloshinLCP2022},
we assume that the environment follows the discounted, labelled Markov Decision Process (MDP) framework given by the tuple $\M = (\SState^\M, \A^\M, P^\M, d_0^\M, \gamma, L^{\mathcal{M}})$ consisting of a state space $\SState^\M$, an action space $\A^\M$, an \textit{\textbf{unknown}} transition function $P^\M : \SState^\M \times \A^\M \to \Delta(\SState^\M)$, an initial state distribution $d^{\mathcal{M}}_0 \in \Delta(\SState^\M)$, and a labelling function $L^{\mathcal{M}}: \SState^\M \to \Sigma$. Let $A^\M(s)$ to be the set of available actions in state $s$. 

Unlike traditional MDPs, $\M$ has a labeling function $L^{\mathcal{M}}$ which returns the atomic propositions that are true in that state. For example, in Figure \ref{fig:motivation} (First), when the agent enters a state $s \in \SState^\M$ such that it is both in the yellow and blue zone then $L^{\mathcal{M}}(s) = \{y, b\}$.


\subsection{Linear Temporal Logic (LTL)} 

Here we give a basic introduction to LTL. For a more comprehensive overview, see \citet{modelchecking}.

\begin{defn}[LTL Specification, $\varphi$]
An LTL specification $\varphi$ is the entire description of the task,  constructed from a composition of atomic propositions with logical connectives: not ($\lnot$), and ($\&$), and implies
($\rightarrow$); and temporal operators: next $(X)$, repeatedly/always/globally ($G$), eventually ($F$), and until ($U$).
\end{defn}

\textbf{Examples.} 
For $AP = \{x, y\}$, some basic task specifications include safety ($G \lnot x$), reachability ($F x $),
stability ($FG x$), response ($x \rightarrow F y$), and progress $(x \;\&\; X F y)$. 

Consider again the environment in Figure \ref{fig:motivation} (First) where $AP = \{y, r, b\}$. If the task is to eventually reach the yellow zone and stay there (known as stabilization) then we write $\varphi = F G y$. Or, if we would like the agent to infinitely loop between the yellow and red zone while avoiding the blue zone then $\varphi = GF(y \;\&\; X F r) \;\&\; G\lnot b$, a combination of safety, reachability, and progress. 

\subsection{LTL Satisfaction}

LTL has recursive semantics defining the meaning for logical connective satisfaction. Without loss of generality, we will be using a specialized automaton, an LDBA $\B_\varphi$ \cite{Sickert2016ldba}, defined below to keep track of the progression of $\varphi$ satisfaction. More details for constructing LDBAs are in \citet{hahn2013lazy, modelchecking, kvretinsky2018owl}. We drop $\varphi$ from $\B_\varphi$ for brevity. 

\begin{defn}(Limit Deterministic B\"uchi Automaton, LDBA \cite{Sickert2016ldba})
\label{def:ldba}
An \textbf{\textit{LDBA}} is a tuple $\B = (\SState^\B, \Sigma \cup \A_\B, P^\B, \SState^{\B^\ast}, b^\B_{-1})$ consisting of (i) a finite set of states $\SState^\B$, (ii) a finite alphabet $\Sigma = 2^{\text{AP}}$, $\A_\B$ is a set of indexed jump transitions (iii) a transition function $P^\B : \SState^\B \times (\Sigma \cup \A_\B) \to \SState^\B$, (iv) accepting states $\SState^{\B^\ast} \subseteq \SState^\B$, and (v) initial state $b^\B_{-1}$. There exists a mutually exclusive partitioning of $\SState^\B = \SState^\B_D \cup \SState^\B_{N}$ such that $\SState^{\B^\ast} \subseteq \SState^\B_D$, and for $b \in S^\B_D, a \in \Sigma$ then $P^\B(b, a) \subseteq \SState^\B_D$, closed. $\A_\B(b)$ is only (possibly) non-empty for $b \in \SState^\B_{N}$ and allows $\B$ to transition to $\SState^\B_{D}$ without reading an AP. 
A \textit{path} $\varrho = (b_0, b_1, \ldots)$ is a sequence of states in $\B$ reached through successive transitions under $P^\B$.
\end{defn}

\begin{defn}($\B$ accepts)
      $\B$ \textbf{accepts} a path $\varrho$ if there exists some state $b \in \SState^{\B^\ast}$ in the path that is visited infinitely often.
\end{defn}

\textbf{Examples.} 
Consider again the environment in Figure \ref{fig:motivation} (First) where $AP = \{y, r, b\}$. If we would like to make an LDBA for $\varphi = F G y$ (reach and stabilize at $y$) then we would get the state machine seen in Figure \ref{fig:motivation} (Second). In this state machine, the agent starts at state $0$. The accepting set is given by $\SState^{\B^\ast} = \{1\}$. The transition between state $0$ and state $1$ is what is formally referred to as a jump transition: $\A_\B(0) = \{\epsilon\}$ while $\A_\B(\cdot) = \varnothing$ otherwise. Whenever the agent is in state $0$ of the LDBA, there is a choice of whether to stay at state $0$ or transition immediately to state $1$. This choice amounts to the agent believing that it has satisfied the 
``eventually'' part of the LTL specification. When the agent takes this jump, then it must thereafter satisfy $y$ to stay in state $1$. The agent gets the decision of when it believes it is capable of satisfying $y$ thereafter. When the agent takes the jump, if it fails to stay in $y$, it immediately transitions to the sink, denoted state $2$. 
The LDBA accepts when the state $1$ is reached infinitely often, meaning the agent satisfies ``always $y$'' eventually, as desired. 

Another example, this time without jump transitions, would be for $\varphi = GF(y \;\&\; X F r) \;\&\; G\lnot b$ (oscillate between $y$ and $r$ forever while avoiding $b$). The LDBA can be seen in Figure \ref{fig:motivation} (Third). In this state machine, the agent starts at state $1$ and the accepting set is given by $\SState^{\B^\ast} = \{1\}$. To make a loop back to state $1$, the agent must visit both $r$ and $y$. Doing so infinitely often satisfies the LDBA condition and therefore the specification. If at any point $b$ is encountered then the agent transitions to the sink, denoted state $3$.

\section{Problem Formulation}

We first introduce slightly more notation. Let $\mathcal{Z} = \SState^\mathcal{M} \times \SState^\B$. Let $\Pi: \mathcal{Z} \times \A \to \Delta([0, 1])$ be a (stochastic) policy class over the product space of the MDP and the LDBA (defined below), where $\A((s,b)) = \A^\mathcal{M}(s) \cup \A^\B(b)$, to account for jump transitions in $\B$. 

\textbf{Synchronizing the MDP with the LDBA.} For any $(s, b) \in \mathcal{Z}$, a policy $\pi \in \Pi$ is able to select an action in $\A^\mathcal{M}(s)$ or an action in $\A^\B(b)$, if available. We can therefore generate a \textbf{trajectory} as the sequence $\tau = (s_0, b_0, a_0, s_1, b_0, a_1, \ldots)$ under a new probabilistic transition relation given by
\begin{multline}\label{eq:relation}
P(s', b'|s, b, a) = \\
\begin{cases} 
    P^{\M}(s, a, s') 
    &a \in A^{\M}(s), b' \in P^\B(b, L(s')) \\
    1, &a \in  A^{\B}(b), b' \in P^\B(b, a), s=s' \\
    0, &\text{otherwise}
    \end{cases}
\end{multline}

Let the LDBA projection of $\tau$ be the subsequence $\tau_{\B} = (b_0, b_1, \ldots)$.  Elements of $\tau_{\B}$ can be thought of as tracking an agent's LTL specification satisfaction: 

\begin{defn}[Run Satisfaction, $\tau \models \varphi$]
We say a trajectory satisfies $\varphi$ if $\B$ accepts $\tau_\B$, which happens if $\exists b \in \tau_{\B}$ infinitely often with $b \in \SState^{\B^\ast}$.
\end{defn}



Let $\Tau^P_\pi = \mathbb{E}_{z \sim d_0^\mathcal{M} \times \{b_{-1}\}}[\Tau^P_\pi(z)]$ be the distribution over all possible trajectories starting from any initial state $z \in d_0^\mathcal{M} \times \{b_{-1}\}$ where $\Tau^P_\pi(z)$ is the (conditional) distribution over all possible trajectories starting from $z \in \mathcal{Z}$ generated by $\pi$ under relation $P$ (given in \eqref{eq:relation}). The probability of LTL satisfaction results from counting how many of the trajectories satisfy the LTL specification:

\begin{defn}[State Satisfaction, $z \models \varphi$]
$\mathbb{P}_{\pi}[z \models \varphi] 
= \mathbb{E}_{\tau \sim \Tau^P_\pi(z)}[\mathbf{1}_{\{\tau \models \varphi\}}] = \mathbb{E}_{\tau \sim \Tau^P_\pi}[\mathbf{1}_{\{\tau \models \varphi\}} | z_0 = z]$ 
\end{defn}

\begin{defn}[Policy Satisfaction, $\pi \models \varphi$]
$\mathbb{P}[\pi \models \varphi] = \mathbb{E}_{\tau \sim \Tau^P_\pi}[\mathbf{1}_{\{\tau \models \varphi\}}]$ where $\mathbf{1}_X$ is the indicator for $X$.
\end{defn}

Ideally we would like to find a policy with highest probability of LTL specification satisfaction: one that generates the most number of LTL-satisfying runs. Formally,
\begin{equation}\label{prob:general}
    \pi^\ast \in \arg\max_{\pi \in \Pi} \mathbb{P}[\pi \models \varphi].
\end{equation}
We note that Eq \eqref{prob:general} is the standard starting point for formulating policy optimization for LTL satisfaction \citep{Yang2021Intractable, Bozkurt2020Q, cai2021modular, Hasanbeig2018lcrl, hasanbeig2020deep, VoloshinLCP2022}.

\section{RL-Friendly Form: Eventual Discounting}
\label{sec:discounting}


Unfortunately, the maximization problem in Eq \eqref{prob:general} is not easily optimized since we dont have a direct signal on $\mathbb{P}[\pi \models \varphi]$. Without any additional assumptions (such as structured knowledge of the MDP), any finite subsequence can only give evidence on whether $\tau \models \varphi$ but not a concrete proof. 

\textbf{Eventual Discounting.}
To address the above issue, we develop a modified value-function based surrogate as follows.
Given a trajectory $\tau = (s_0, b_0, a_0, \ldots)$, we  keep track of how often $b_i \in \SState^{\B^\ast}$ and incentivize an agent to visit $\SState^{\B^\ast}$ as many times as possible. In particular, under eventual discounting, the value function will give the agent a reward of $1$ when in a state $b_i \in \SState^{\B^\ast}$ and not discount length of time between visits to $\SState^{\B^\ast}$. Formally, we will be seeking
\begin{equation}\label{prob:discounted}
    \pi_\gamma^\ast \in \arg\max_{\pi \in \Pi} \mathbb{E}_{\tau \sim \Tau_{\pi}^P}[\sum_{i=0}^\infty \Gamma_i \mathbf{1}_{\{b_i \in \SState^{\B^\ast}\}}]\ \ \  (\equiv V_{\pi}^\gamma),
\end{equation}
where $\Gamma_0 = 1$ and
 \begin{equation}\label{def:gamma}
    \Gamma_i = \prod_{t=0}^{i-1} \gamma(b_t), 
    \quad \gamma(b_t) = 
    \begin{cases}
    \gamma, \quad b_t \in \SState^{\B^\ast} \\
    1, \quad \text{otherwise}
    \end{cases}.
\end{equation}

\textbf{Intuition for $\Gamma_i$.} At first glance setting $\Gamma_i = \gamma^i$ to be the traditional RL exponential discount rate would seem reasonable. 
Unfortunately, $\nexists \gamma \in [0, 1]$ with $\Gamma_i = \gamma^i$ that avoids catastrophic myopic behavior. In particular, take Figure \ref{fig:motivation} (Fourth). The agent starts in state $0$ and only has two actions $A$ and $B$. Taking action $A$ transitions directly to an accepting state from which point the accepting state is visited every $2$ steps. On the other hand, action $B$ transitions to an accepting state with probability $\alpha$ and a sink state with probability $1-\alpha$. The accepting state reached by action $B$ is revisited every step. Suppose $\beta = \pi(A) = 1-\pi(B)$ then we can calculate:
\begin{equation}
\mathbb{E}_{\tau \sim \Tau_{\pi}^P}[\sum_{i=0}^\infty \gamma^i \mathbf{1}_{\{b_i \in \SState^{\B^\ast}\}}] =  \frac{\beta}{1-\gamma^2} +  \frac{(1-\beta)\alpha}{1-\gamma}.
\end{equation}
For $\alpha > 1/2$, the optimal choice $\beta$ is $\beta = 0$ implying that $P(\pi \models \varphi) = \alpha$. When $\alpha \in (1/2,1)$ then this implies that $\pi$ is not probability optimal. Indeed, $P(\pi \models \varphi) = \alpha < 1$ when $\beta = 0$ but $P(\pi^\ast \models \varphi) = 1$ by selecting $\beta = 1$. The intuition here, which can be formalized by taking $\gamma \to 1$, is that the average reward for taking action $A$ is $\frac{1}{2}$ while the average reward for taking action $B$ is $1$ with probability $\alpha$, which is worth the risk for large $\alpha > 1/2$.

To avoid this myopic behavior, we must avoid discriminating between return times between good states. The number steps (on average) it takes to return to $\SState^{\B^\ast}$ is irrelevant: we only require that the system does return. For this reason we do not count time (hence $\gamma = 1$) in our definition of $\Gamma_i$ when the system is not in $\SState^{\B^\ast}$. We call this \emph{eventual discounting}.



\subsection{Analysis of $\pi^\ast_\gamma$}

In this section we analyze how the probability of $\pi_\gamma^\ast$ satisfying $\varphi$ compares to that of the best possible one $\pi^\ast$.

Let the set $O(\tau) = \{i: b_i \in \SState^{\B^\ast}\}$ denote the occurences (time steps) when a good state is reached. This quantity is natural since $|O(\tau)| = \infty$ if and only if $\tau \models \varphi$. 

\begin{lem}\label{lem:bound} For any $\pi \in \Pi$ and $\gamma \in (0, 1)$, we have
\begin{equation*}
|(1-\gamma) V_{\pi}^\gamma - \mathbb{P}[\pi \models \varphi]| \leq \log(\frac{1}{\gamma}) O_\pi
\end{equation*}
where $O_\pi = \mathbb{E}_{\tau \sim \Tau_{\pi}^P}\left[|O(\tau)|\bigg| 
\tau \not\models \varphi \right]$ is the expected number of visits to an accepting state for the trajectories that do not satisfy $\varphi$.
\end{lem}

\begin{proof}
Fix some state $z = (s, b) \in  \mathcal{Z}$.
\begin{align*}
V_{\pi}^\gamma(z) &= \mathbb{E}_{\tau \sim \Tau_{\pi}^P}[\sum_{i=0}^\infty \Gamma_i \mathbf{1}_{\{b_i \in \SState^{\B^\ast}\}} | z_0 = z] \\
&= \mathbb{E}_{\tau \sim \Tau_{\pi}^P}\left[\sum_{j = 0}^{|O(\tau)|} \gamma^j | z_0 = z \right] 
\end{align*}
Using the fact that $\sum_{j=0}^k \gamma^j = \frac{1-\gamma^k}{1-\gamma}$, we have
\begin{multline}\label{eq:expansion}
V_{\pi}^\gamma(z) = \mathbb{E}_{\tau \sim \Tau_{\pi}^P}\left[\frac{1-\gamma^{|O(\tau)|}}{1-\gamma}\bigg| \substack{
\tau \models \varphi \\
z_0 = z \\
}
 \right] \mathbb{P}_{\pi}[z \models \varphi] \\
 +\mathbb{E}_{\tau \sim \Tau_{\pi}^P}\left[\frac{1-\gamma^{|O(\tau)|}}{1-\gamma}\bigg| \substack{
\tau \not\models \varphi \\
z_0 = z \\
} \right] \mathbb{P}_{\pi}[z \not\models \varphi] 
\end{multline}
Since $|O(\tau)| = \infty$ for any $\tau \models \varphi$, 
\begin{equation}\label{eq:equality}
    \mathbb{E}_{\tau \sim \Tau_{\pi}^P}\left[\frac{1-\gamma^{|O(\tau)|}}{1-\gamma}\bigg| \substack{
\tau \models \varphi \\
z_0 = z \\
}
 \right] = \frac{1}{1-\gamma} 
\end{equation}
together with $\mathbb{P}_{\pi}[z \not\models \varphi] \geq 0$ implies
\begin{equation}
V_{\pi}^\gamma(z) \geq \frac{1}{1-\gamma} \mathbb{P}_{\pi}[z \models \varphi].
\end{equation}
Taking the expectation over initial states we have
\begin{equation}\label{eq:lower_bound}
V_{\pi}^\gamma \geq \frac{1}{1-\gamma} \mathbb{P}[\pi \models \varphi].
\end{equation}
Now we find an upper bound. Let $M_\pi(t) = \mathbb{E}_{\tau \sim \Tau_{\pi}^P}\left[e^{ t |O(\tau)|}\bigg| 
\tau \not\models \varphi \right]$. Starting again with Eq \eqref{eq:expansion} and using Eq \eqref{eq:equality}, we have
\begin{equation}
V_{\pi}^\gamma(z) 
\leq \frac{\mathbb{P}_{\pi}[z \models \varphi]}{1-\gamma}  + \frac{1- \mathbb{E}_{\tau \sim \Tau_{\pi}^P}\left[e^{ \log(\gamma) |O(\tau)|}\bigg| \substack{
\tau \not\models \varphi \\
z_0 = z \\
} \right] }{1-\gamma}
\end{equation}
where we have used that $\mathbb{P}_{\pi}[z \not\models \varphi] \leq 1$
for any $z \in \mathcal{Z}$. Taking the expectation with respect to the initial state distribution then we have
\begin{equation}\label{eq:upperbound2}
(1-\gamma) V_{\pi}^\gamma \leq \mathbb{P}[\pi \models \varphi]  + 1 - M_\pi(\log(\gamma))
\end{equation}
In particular, $M_{\pi}(t)$ is convex and therefore it lies above its tangents: 
\begin{align*}
M_{\pi}(t) \geq M_{\pi}(0) + t M'_{\pi}(0) &= 1 + t  \mathbb{E}_{\tau \sim \Tau_{\pi}^P}\left[|O(\tau)|\bigg| 
\tau \not\models \varphi \right] \\
&= 1 + t O_{\pi}
\end{align*} 
Plugging this inequality into Eq \eqref{eq:upperbound2}, together with Eq \eqref{eq:lower_bound},
\begin{equation}\label{eq:upperbound3}
\mathbb{P}[\pi \models \varphi] \leq (1-\gamma) V_{\pi}^\gamma \leq \mathbb{P}[\pi \models \varphi]  + \log(\frac{1}{\gamma}) O_\pi
\end{equation}
Subtracting $\mathbb{P}[\pi \models \varphi]$ from both sides and taking the absolute value completes the proof.
\end{proof}

\begin{thm}
\label{thm:finite_bound}(Non-asymptotic guarantee) For any $\gamma \in (0, 1)$,
\begin{equation}
\sup_{\pi \in \Pi} \mathbb{P}[\pi \models \varphi] - \mathbb{P}[\pi_\gamma^\ast \models \varphi] \leq 2 \log(\frac{1}{\gamma}) \sup_{\pi \in \Pi} O_\pi
\end{equation}
where $O_\pi = \mathbb{E}_{\tau \sim \Tau_{\pi}^P}\left[|O(\tau)|\bigg| 
\tau \not\models \varphi \right]$.
\end{thm}

\begin{proof}


Consider any sequence $\{\pi_i\}_{i=1}^\infty$ such that $\mathbb{P}[\pi_i \models \varphi] \to \sup_{\pi} \mathbb{P}[\pi \models \varphi]$ as $i \to \infty$. Then we have for any $\pi_i$,
\begin{align*}
\mathbb{P}[\pi_i \models \varphi] - \mathbb{P}[\pi_\gamma^\ast \models \varphi] &= \mathbb{P}[\pi_i \models \varphi] - (1-\gamma) V_{\pi_i}^\gamma \\
&\quad+ (1-\gamma)V_{\pi_i}^\gamma - (1-\gamma)V_{\pi_\gamma^\ast}^\gamma \\
&\quad+ (1-\gamma) V_{\pi_\gamma^\ast}^\gamma - \mathbb{P}[\pi_\gamma^\ast \models \varphi] \\
&\stackrel{(a)}{\leq} |\mathbb{P}[\pi_i \models \varphi] - (1-\gamma) V_{\pi_i}^\gamma| \\
&\quad+ |\mathbb{P}[\pi_\gamma^\ast \models \varphi] - (1-\gamma) V_{\pi_\gamma^\ast}^\gamma|
\\
&\stackrel{(b)}{\leq}  \log(\frac{1}{\gamma})(O_{\pi_i} + O_{\pi_\gamma^\ast}) \\
&\stackrel{(c)}{\leq} 2 \log(\frac{1}{\gamma}) \sup_{\pi \in \Pi} O_\pi
\end{align*}
where $(a)$ is triangle inequality together with removing the term $(1-\gamma)V_{\pi_i}^\gamma - (1-\gamma)V_{\pi_\gamma^\ast}^\gamma$ since it is nonpositive by definition of $\pi_\gamma^\ast$, $(b)$ is an application of Lemma \ref{lem:bound}, and $(c)$ is a supremum over all policies. Taking the limit on both sides as $i \to \infty$ completes the proof.
\end{proof}


\begin{cor}\label{cor:finite}
If the number of policies in $\Pi$ is finite then $\sup_{\pi \in \Pi} O_\pi = m < \infty$ is attained, is a finite constant and
 $$
 \sup_{\pi \in \Pi} \mathbb{P}[\pi \models \varphi] - \mathbb{P}[\pi_\gamma^\ast \models \varphi] \leq 2 m \log(\frac{1}{\gamma})
 $$
\end{cor}

\begin{cor}\label{cor:deterministic}
In the case that $\SState^\mathcal{M}$ and $\A^\mathcal{M}$ are finite, 
 then $\mathcal{Z}$ and $\A$ are finite. It is known that optimal policies are deterministic \cite{puterman2014markov} and therefore there we need only consider deterministic policies, for which there are a finite number. Thus $\sup_{\pi \in \Pi} O_\pi = m < \infty$ is attained, is a finite constant and
 $$
 \sup_{\pi \in \Pi} \mathbb{P}[\pi \models \varphi] - \mathbb{P}[\pi_\gamma^\ast \models \varphi] \leq 2 m \log(\frac{1}{\gamma})
 $$
\end{cor}



\subsection{Interpretation} \label{sec:interpretation}
Theorem \ref{thm:finite_bound} relies on the quantity $\sup_{\pi \in \Pi} O_\pi$ to be finite for the bound to have meaning. In fact, we need only make requirements on $M_\pi(\log(\gamma))$ but the requirements are more easily understood on $O_\pi$. As an aside, $M_\pi(\log(\gamma))$ can be interpreted as the moment generating function of the random variable which is the number of visits to $\SState^{\B^\ast}$. Instead we consider the equally natural quantity $O_\pi$. $O_\pi$ is the (average) number of times that a good state is visited by a trajectory that does not satisfy the specification. Ideally, this number would be small and it would be easy to discriminate against good and bad policies. 

\textbf{The bad news.} In the case that $\Pi$ is an infinite class, conditions for ensuring $\sup_{\pi \in \Pi} O_\pi$ is finite is nontrivial and is dependent on the landscape of the transition function $P$ of the MDP and $\Pi$. 

Let us suppose $\sup_{\pi \in \Pi} O_\pi$ is infinite. This means there are policies that induce bad trajectories that eventually fail to reach $\SState^{\B^\ast}$, but along the way visited $\SState^{\B^\ast}$ an arbitrarily large (but finite) number of times. 
In other words, they are policies that are indistinguishable from actual probability-optimal policies until the heat death of the universe. 

Consider the specification in Figure \ref{fig:motivation} (right), given by infinitely often cycle between red and yellow while avoiding blue. A good-looking bad policy is one that accomplishes the task frequently but, amongst the times that it fails, it would cycle between red and yellow many times before failing. $\sup_{\pi \in \Pi} O_\pi$ being infinite means that there are policies that will cycle arbitrarily many times before failing. 


\textbf{The good news.} Corollary \ref{cor:deterministic} reveals that discretization suffices to generate probability optimal policies, with suboptimality shrinking at a rate of $\log(\frac{1}{\gamma})$. This suggests that compactness of $P$ and $\Pi$ and continuity of $P$ may very well be enough but we leave these conditions for future work. Finally, since all computers deal with finite precision, the number of policies is finite and therefore Corollary \ref{cor:finite} similarly applies.

\section{LTL Counterfactual Experience Replay}


One can optimize the formulation in Eq \eqref{prob:discounted} using any Q-learning or policy gradient approach, as seen in Algorithm \ref{algo:main} (Line 4). However, doing so is challenging since it suffers from reward sparsity: the agent only receives a signal if it reaches a good state. 

\begin{algorithm}[t]
	\caption{$\texttt{Learning with LCER}$ } 
	\label{algo:main}
	\begin{algorithmic}[1]
            \REQUIRE Maximum horizon $T$. Replay buffer $D = \{\}$.
            \FOR{$k = 1, 2, \ldots$}
                \STATE Run $\pi_{k-1}$ in the MDP for $T$ timesteps and collect
                $\tau = (s_0, b_0, a_0,\ldots,s_{T-1}, b_{T-1}, a_{T-1}, s_{T}, b_{T})$
                \STATE $D_k \leftarrow \texttt{LCER}(D_{k-1}, \tau)$ 
                \STATE $\pi_{k} \leftarrow \texttt{Update}(\pi_{k-1}, D_k)$ \COMMENT{Q-learn/Policy grad.} \\
            \ENDFOR
	\end{algorithmic}
\end{algorithm}

We combat reward sparsity by exploiting the LDBA: $P^\mathcal{\B}$ is completely known. By knowing  $P^\mathcal{\B}$, we can generate multiple off-policy trajectories from a single on-policy trajectory by modifying which stats in the LDBA we start in, which notably does not require any access to the MDP transition function $P^\mathcal{M}$. We call this approach \emph{LTL-guided Counterfactual Experience Replay}, $\texttt{LCER}$ (Algorithm \ref{algo:main}, Line 3), as it is a modification of standard experience replay \citep{lin1992self,mnih2013playing,mnih2015human} to include counterfactual experiences elsewhere in the LDBA. $\texttt{LCER}$ is most simply understood through Q-learning, and needs careful modification for policy gradient methods.

\textbf{Q-learning with $\texttt{LCER}$.} 
See Algorithm \ref{algo:q_learning} for a synopsis of $\texttt{LCER}$ for Q-learning. Regardless of whatever state $s \in \SState^\mathcal{M}$ the agent is in, we can pretend that the agent is in any $b \in \SState^{\B}$. Then for any action the agent takes we can store experience tuples:
\begin{equation}\label{q-learn:all_pairs}
    \{(s, b, a, r, s', \tilde{b'}) \;|\;\; \forall b \in \SState^\B\}
\end{equation}
where $\tilde{b'} = P^{\B}(b, L^\mathcal{M}(s'))$ is the transition that would have occurred from observing labelled state $L(s')$ in state $(s, b)$ and $r = \mathbf{1}_{\tilde{b}' \in B^\ast}$. Furthermore we can add all jump transitions:
\begin{equation}
    \{(s, b, \epsilon, r, s, \tilde{b'}) \;|\;\; \forall b  \in \SState^\B, \forall\epsilon \in \A^\B(b) \}
\end{equation}
since jumps also do not affect the MDP. Notice when we add the jumps that $s' = s$, since only the LDBA state shifts in a jump.


\begin{algorithm}[t]
	\caption{$\texttt{LCER}$ for Q-learning} 
	\label{algo:q_learning}
	\begin{algorithmic}[1]
            \REQUIRE Dataset $D$. Trajectory $\tau$ of length $T$.
            \FOR{$(s_t, a_t, s_{t+1}) \in \tau$}
                \FOR{$b \in \SState^{\B}$}
                    \STATE Set $\tilde{b} \leftarrow P^{\B}(b, L^\mathcal{M}(s_{t+1}))$
                    \STATE $D \leftarrow D \cup (s_t, b, a_t, \mathbf{1}_{\tilde{b} \in B^\ast}, s_{t+1}, \tilde{b})$
                    \FOR{$\epsilon \in \mathcal{A}^{\B}(s)$}
                        \STATE Set $\tilde{b} \leftarrow P^{\B}(b, \epsilon)$
                        \STATE $D \leftarrow D \cup (s_t, b, \epsilon, \mathbf{1}_{\tilde{b} \in B^\ast}, s_t, \tilde{b})$
                    \ENDFOR
                \ENDFOR 
            \ENDFOR
            \RETURN{$D$}
	\end{algorithmic}
\end{algorithm}

\textbf{Policy Gradient with $\texttt{LCER}$.}  See Algorithm \ref{algo:pg} for a summary of $\texttt{LCER}$ for policy gradient. For policy gradient, unlike Q-learning, it is necessary to calculate future reward-to-go: $R_k(\tau) = \sum_{i=k}^T \Gamma_i \mathbf{1}_{\{b_i \in \SState^{\B^\ast}\}}$. Thus, we have to generate entire trajectories that are consistent with $P^\B$ rather than independent transition tuples as in Eq \eqref{q-learn:all_pairs}. We will show how to generate all feasible trajectories.

Consider a trajectory $\tau = (s_0, b_0, a_0,\ldots, s_{T}, b_{T})$ was collected. 
Let us remove jump transitions $(s_i, b_i, a_i)$ where $a_i \in \mathcal{A}^\B(b_i)$ and consider the projection of the trajectory to the MDP $\tau_\M = (s_0, s_1, \ldots, s_T)$.
We should only have control over the initial LDBA state $b_0$ as all other automaton states $(b_1, \ldots, b_T)$ in a trajectory sequence are determined by $\tau_{\mathcal{M}}$ and $b_{i+1} = P^\B(b_{i}, L^{\mathcal{M}}(s_i))$.

Therefore we add 
\begin{multline}
    \tilde{\mathcal{T}}(\tau) = \{(s_0, \tilde{b}_0, a_0,\ldots, s_{T}, \tilde{b}_{T}) \;|\;\; \\
    \forall \tilde{b}_0  \in \SState^\B, \tilde{b}_i = P^\B(\tilde{b}_{i-1}, L^{\mathcal{M}}(s_{i}))  \}
\end{multline}
where only the LDBA states are different between the trajectories.

Now we handle jump transitions. Consider some $\tilde{\tau} \in \tilde{\mathcal{T}}(\tau)$. Recall, a jump transition can occur whenever $\mathcal{A}^\B(\tilde{b}_i)$ is non-empty. This involves adding a trajectory that is identical to  $\tilde{\tau}$ all the way until the jump occurs. The jump occurs and then the same action sequence and MDP state sequence follows but with different LDBA states. Specifically, suppose $\tilde{b_i}$ had an available jump transitions, $\epsilon \in \mathcal{A}^\B(\tilde{b}_i)$. Then:
\begin{equation}\label{eq:jumps}
    \tilde{\tau}_{i, \epsilon} = (s_0,\tilde{b}'_{i},a_0, \ldots, s_{i}, \tilde{b}'_{i}, \epsilon, s_i, \tilde{b}'_{i+1}, a_{i}, \ldots, s_T, \tilde{b}'_T) 
\end{equation}
where $\tilde{b}'_k = \tilde{b}_i$ for $k \leq i$ and $\tilde{b}'_k = P^\B(\tilde{b}'_{k-1}, L^{\mathcal{M}}(s_{k}))$ otherwise. 

We have to add all possible $\tilde{\tau}'_{i, \epsilon}$ that exist. Let $\mathcal{E}$ be the operator that adds jumps to existing sequences:
\begin{multline}
    \mathcal{E}(\tilde{\mathcal{T}}(\tau)) = \tilde{\mathcal{T}}(\tau) \cup \{ \tilde{\tau}_{i, \epsilon} \text{ from Eq } \eqref{eq:jumps} | \\
    \forall \tilde{\tau} \in \tilde{\mathcal{T}}(\tau), \exists b_i \in \tilde{\tau} \text{ s.t. } \exists \epsilon \in \mathcal{A}^\B(b_i) \}.
\end{multline}
We can only apply $\mathcal{E}(\mathcal{E}(\ldots(\mathcal{E}(\tilde{\mathcal{T}}(\tau)))))$ at most $T$ times since the original length of $\tau$ is $T$.

\begin{algorithm}[t]
	\caption{$\texttt{LCER}$ for Policy Gradient} 
	\label{algo:pg}
	\begin{algorithmic}[1]
            \REQUIRE Dataset $D$. Trajectory $\tau$ of length $T$.
            \STATE Set $\mathcal{\tilde{T}}_0 \leftarrow \mathcal{\tilde{T}}(\tau)$
            \FOR{$k = 1,\ldots, T-1$} 
                \STATE $\mathcal{\tilde{T}}_k \leftarrow \mathcal{E}(\mathcal{\tilde{T}}_{k-1})$
                \IF{$\mathcal{\tilde{T}}_k == \mathcal{\tilde{T}}_{k-1}$}
                \STATE Set $\mathcal{\tilde{T}}_{T-1} \leftarrow \mathcal{\tilde{T}}_k$
                \STATE \textbf{break}
                \ENDIF
            \ENDFOR
            \STATE Set $D \leftarrow D \cup \mathcal{\tilde{T}}_{T-1}$
            \RETURN{$D$}
	\end{algorithmic}
\end{algorithm}

\begin{rem}
The length of $\tau$ has to be sufficiently large to make sure the LDBA has opportunity to reach $\SState^{\B^\ast}$. A sufficient condition is $T \geq |\{b | \A^\B(b) \neq \varnothing\}|$, the number of LDBA states with jump transitions.
\end{rem}

It is possible to constructively generate feasible trajectories during the rollout of a policy rather than after-the-fact, see Appendix \ref{app:LCER-pg}.


\begin{figure*}[!ht]
  \includegraphics[width=1\linewidth]{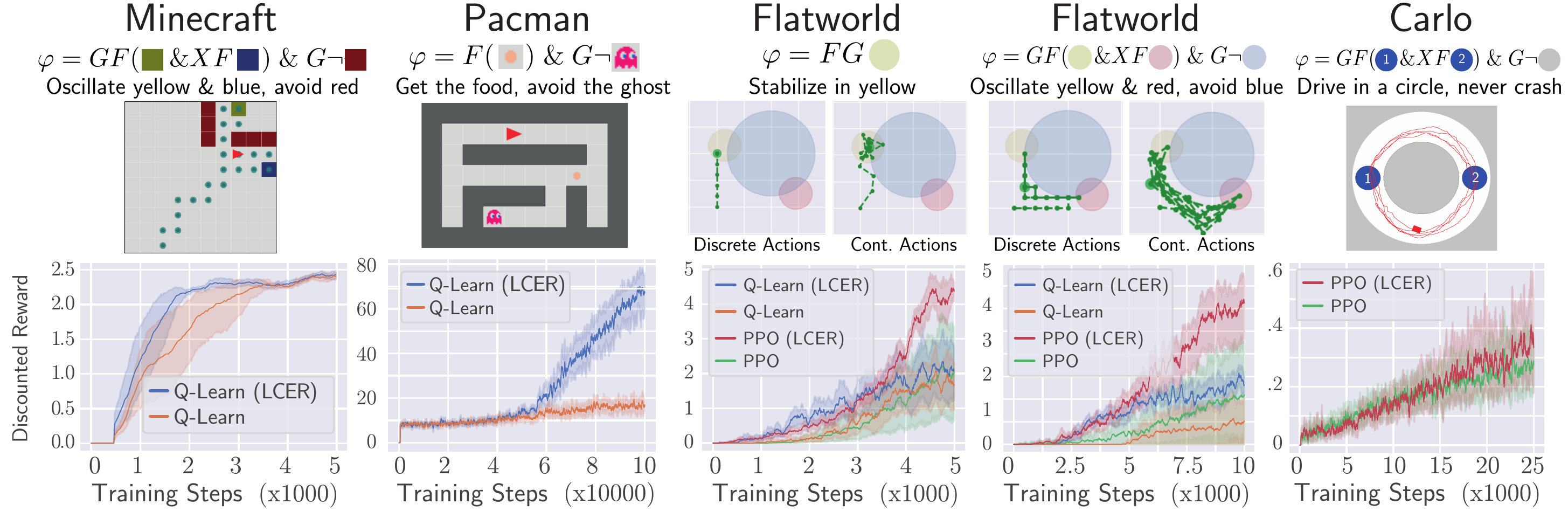}
\caption{\textit{Results.} Each column is an environment and a LTL formula we'd like an agent to satisfy. The environment and a trajectory from the final policy is illustrated in the center of the column (except for Pacman, which is the initial state). The learning curves at the bottom of each column show that adding off-policy data using $\texttt{LCER}$ has strong benefits for empirical performance. \textit{\textbf{First Column:}} Minecraft, where an agent should visit the yellow and blue areas while avoiding the red. The final policy is illustrated via blue dots. \textit{\textbf{Second Column:}} Pacman, where an agent should collect the food while avoiding a ghost. \textit{\textbf{Third Column:}} Flatword, where an agent should eventually stabilize in the yellow region. When the actions are discrete we use Q-learning, when the actions are continuous we use PPO. \textit{\textbf{Fourth Column:}} Same as the third column except an agent should oscillate between the yellow and red regions while avoiding the blue. \textit{\textbf{Fifth Column:}} Carlo, where an agent should drive in a circle without crashing by visiting the blue regions labelled $1$ and $2$ infinitely often.
}
 \label{fig:results}
 \vspace{-.1in}
\end{figure*}

\section{Experiments}


We perform experiments in four domains with varying LTL formulas, state spaces, action spaces, and environment stochasticity summarized in the following section. Our aim is to answer the following two questions: (1) Can we achieve policies that behave the way we expect an LTL-satisfying policy to behave? (2) How does $\texttt{LCER}$ impact the performance of learning. 

\subsection{Environment Details}
\textbf{Minecraft} The Minecraft environment is a $10 \times 10$ deterministic gridworld with 5 available actions: left, right, up, down, nothing. The agent, given by a red triangle starts in the cell $(9,2)$. 
The environment, as well as the final behavior of the agent (given by blue dots) can be seen in Figure \ref{fig:results} (First).

\textbf{Pacman} The Pacman environment is a $5 \times 8$ deterministic gridworld with 5 available actions: left, right, up, down, nothing. The agent, given by a red triangle starts in the cell $(0, 3)$. 
The ghost chases the agent with probability $0.8$ and takes a random action with probability $0.2$, for this reason the environment is stochastic. The starting position of the environment can be seen in Figure \ref{fig:results} (Second).

\textbf{Flatworld} The Flatworld environment (seen in Figure \ref{fig:results} Third and Fourth) is a two dimensional continuous world. The agent (given by a green dot) starts at $(-1, -1)$. The dynamics of the world are given by $x' = x + a/10$ where both $x \in \mathbb{R}^2$ and $a \in [0,1]^2$. We also allow the action space to be discrete by letting there be $5$ actions (right, up, left, down, nothing) where the agent takes a full-throttle action in each respective direction. 
 
\textbf{Carlo} The Carlo environment (seen in Figure \ref{fig:results} Fifth)is a simplified self-driving simulator that uses a bicycle model for the dynamics. The agent observes its position, velocity, and heading in radians for a total of $5$ dimensions. The agent has control over its heading and throttle, for an action space of $[-1, 1]^2$. For this domain, we have chosen to use a circular track where the agent starts in the center of the road at an angle of $\{\pi(1+2i)/4\}_{i=0}^3$ and drive counterclockwise around in a circle without crashing.

\subsection{Methods and Baseline}

When the action space is discrete, we use Q-learning with $\texttt{LCER}$ otherwise we use PPO with $\texttt{LCER}$. The baseline we compare against is the same method without $\texttt{LCER}$. This allows us to verify the extent to which $\texttt{LCER}$ impacts performance. We also plot a trajectory from the final policy for each environment in the middle of each column of Figure \ref{fig:results}, except for Pacman as it is difficult to visualize the interaction between the ghost and pacman outside of video.

\textbf{Dealing with $\A$.} For PPO, the agent's policy is a Gaussian (as in standard implementations) over the continuous action space. In order to deal with jump transitions (in the LDBA) when in a continuous action space (in the MDP), we first let the agent decide whether to execute a jump transition or not (ie. a probabilistic coin flip). If the agent chooses to not, then we take the action according to the Gaussian. The coin flip probability is learned, as well as the Gaussian. For the importance sampling term of PPO, the density of $\pi$ is modified to account for the coin flip. For more details see Appendix \ref{app:experiments}.

\subsection{Results}

\textbf{Can we achieve desired behavior?} The answer here is a resounding yes. For each environment (except Pacman) we illustrate the trajectory of the final policy above each learning curve in Figure \ref{fig:results}. Determining the probability of satisfaction of the final policy is currently a challenging open problem (except in finite-state action spaces). Nevertheless, in each environment the agent qualitatively accomplishes the task. Even for challenging tasks with continuous action spaces, the agent is able to learn to accomplish the LTL specification.

\textbf{Does $\texttt{LCER}$ help in the learning process?} According to the learning curves in the last row of Figure \ref{fig:results}, $\texttt{LCER}$ demonstrably expedites learning. In every environment with the exception of Carlo, $\texttt{LCER}$ generates significant lift over lack of experience replay.

\textbf{Intuition for why $\texttt{LCER}$ helps?} One way of viewing an LDBA is as a curriculum for what steps need to be taken in order to accomplish a task. By replacing the LDBA state of the agent with some other dream LDBA state, we are allowing the agent to ``pretend'' that it has already accomplished some portion of the task. 

As an example, consider the Flatworld example in Figure \ref{fig:results} with $\varphi = GF(y \;\&\; XF(r)) \;\&\; (G\lnot b)$. A baseline agent (without $\texttt{LCER}$) would need to accomplish the entirety of the task in order to see any reward. However, an agent with counterfacual data, need only visit $y$ from state $0$ of the LDBA (see figure \ref{fig:motivation} for the LDBA). Then once the agent is really good at getting to $y$, it needs to learn how to reach $r$ from state $2$. After both of these tasks are accomplished, independently, the agent has solved the whole task. By placing the agent in state $0$ of the LDBA, we are effectively letting the agent pretend that it has already visited $r$. In this sense, part of the task has been accomplished.

\section{Related Work}\label{sec:related-work} 

\textbf{Finding LTL-satisfying policies.} Among the attempts at finding LTL-satisfying policies, Q-learning approaches have been the primary method of choice when the dynamics are unknown and Linear Programming methods when the dynamics are known \cite{Sadigh2014, Hasanbeig2018lcrl, Bozkurt2020Q, Cai2021, Ding2014}. The Q-learning approaches are predominantly constrained to finite state-action spaces. Among the works that extend to continuous action spaces \cite{hasanbeig2020deep}, DDPG is used and takes the form of hierarchical RL which is known to potentially find myopic policies \cite{Icarte2020RewardMachine}. 

Handling a subset of LTL specifications involving those expressible as finite expressions can also be addressed with Reward machines \cite{Icarte2020RewardMachine, Camacho2019LTLAndBeyond, Vaezipoor2021LTL2Action}. Our work handles $\omega$-regular expressions, subsuming regular expressions. Many problems are $\omega$-regular problems, but not regular, such as liveness (something good will happen eventually) and safety (nothing bad will happen forever). 

\textbf{On the formulation in Eq \eqref{prob:discounted}.} Notable prior work on defining the value function as a function of the number of visits to $\SState^\B$ and a state-dependent $\Gamma_i$ function include \citet{Bozkurt2020Q, cai2021modular}. Most notably, these authors use multiple different state-dependent discount rates that have a complicated relationships between them that needs to be satisfied in the limit as $\gamma \to 1^-$. Our work drastically simplifies this, getting rid of the technical assumptions, while strengthening the guarantees. This allows us to find a non-asymptotic dependence on the suboptimality of a policies' probability of LTL satisfaction as a function of $\gamma$. 

\textbf{Off-policy data.} One may view the counterfactual samples in \citet{Icarte2020RewardMachine} as an instantiation of LCER, limited to finite LTL expressions and discrete action spaces. Extension to continuous action space and full LTL requires a careful treatment. In the continuous action and full LTL setting, \cite{Wang2020ACLTL} incorporate starting the agent from a different initial LDBA state (than $b_{-1}$) which is still on-policy but from a different starting state and doesn't take advantage of the entire LDBA structure. This work can be seen as complimentary to our own. 

\textbf{Theory.} Works with strong theoretical guarantees on policy satisfaction include \citet{FuLTLPAC, Wolff2012RobustControl, VoloshinLCP2022} but are once again limited to discrete state/action spaces. Extensions of these work to continuous state space is not trivial as they make heavy use of the discrete Markov chain structure afforded to them.

\section{Discussion}\label{sec:discussion}

Our work, to the best of our knowledge, is the first to make full use of the LDBA as a form of experience replay and first to use policy gradient to learn LTL-satisfying policies. Our eventual discounting formulation is unrestricted to Finitary fragments of LTL like most prior work.

Despite the guarantees afforded to us by eventual discounting, in general the problem given in Eq \eqref{prob:general} is not PAC learnable \cite{Yang2021Intractable}. Though, like SAT solvers, it is still useful to find reasonable heuristics to problems that are difficult. We show that under particular circumstances, eventual discounting gives a signal on the quantity of interest in \eqref{prob:general} and even when it fails, it selects a policy that is difficult to differentiate from a successful one. Further, the bad news discussed in Section \ref{sec:interpretation} we speculate is unavoidable in general LTL specifications, without significant assumptions on the MDP. For example, for stability problems in LTL and assuming control-affine dynamics then Lyapunov functions can serve as certificates for a policies' LTL satisfaction. A reasonable relaxation to this would be require a system to behave a certain way for a long, but finite amount of time. 



 

\newpage
\bibliography{icml2023}
\bibliographystyle{icml2023}


\onecolumn
\appendix
\section{Experiments}\label{app:experiments}

\subsection{Environment Details}
The environment and experiment details are summarized in Table \ref{tab:experiments}. 

\begin{table*}[!ht]
\caption{Environment Details}
\label{tab:experiments}
\begin{center}
\begin{small}
\begin{tabular}{llllll}
\toprule
Environment & Experiment &  $\SState^{\mathcal{M}}$ & $\A^{\mathcal{M}}$ & Dynamics & LTL Formula \\
 \midrule
Minecraft & Q-learning & Discrete & Discrete & Deterministic & $GF(y \;\&\; XF(b)) \;\&\; (G\lnot r)$ \\
Pacman & Q-learning & Discrete & Discrete & Stochastic & $F(\text{food}) \;\&\; (G\lnot \text{ghost})$ \\
Flatworld 1 & Q-learning & $\mathbb{R}^2$ & Discrete & Deterministic & $FGy$ \\
Flatworld 2 & Q-learning & $\mathbb{R}^2$ & Discrete & Deterministic & $GF(y \;\&\; XF(r)) \;\&\; (G\lnot b)$ \\
Flatworld 3 & PPO & $\mathbb{R}^2$ & $[0,1]^2$ & Deterministic & $FGy$ \\
Flatworld 4 & PPO & $\mathbb{R}^2$ & $[0,1]^2$ & Deterministic & $GF(y \;\&\; XF(r)) \;\&\; (G\lnot b)$ \\
Carlo & PPO & $\mathbb{R}^5$ & $[-1,1]^2$ & Deterministic & $GF(\text{zone1} \;\&\; XF(\text{zone2})) \;\&\; (G\lnot \text{crash})$ \\
 \bottomrule
\end{tabular}
\end{small}
\end{center}
 \end{table*}

\subsection{Experiment Setup}\label{sec:app:details}

Each experiment is run with $10$ random seeds. Results from Figure \ref{fig:results} are from an average over the seeds.

\textbf{Q-learning experiments.} Let $k$ be the greatest number of jump transitions available in some LDBA state $k = \max_{b \in \SState^{\B}} |\A^\B(b)|$. Let $m = \max_{s \in \SState^\mathcal{M}}|\A^\mathcal{M}(s)|$. The neural network $Q_\theta(s)$ takes as input $s \in \SState^{\mathcal{M}}$ and outputs $\R^{(m+k) \times |\SState^\B|}$ a $(m+k)$-dim vector for each $b \in \SState^{\B}$. For our purposes, we consider $Q_\theta(s, b)$ to be the single $(m+k)$-dim vector cooresponding to the particular current state of the LDBA $b$.

When $\SState^{\mathcal{M}}$ is discrete then we parametrize $Q_{\theta}(s,b)$ as a table. Otherwise, $Q_\theta(s,b)$ is parameterized by $3$ linear layers with hidden dimension $128$ with intermediary ReLU activations and no final activation. After masking for how many jump transitions exist in $b$, we can select $\arg\max_{i \in [0,\ldots, |\A^\B(b)|]} Q_\theta(s,b)_i$ the highest $Q$-value with probability $1-\eta$ and uniform with $\eta$ probability. Here, $\eta$ is initialized to $\eta_0$ and decays linearly (or exponentially) at some specified frequency (see Table \ref{tab:hyperparam_q}).

At each episode (after a rollout of length $T$), we perform $K$ gradient steps with different batches of size given in Table \ref{tab:hyperparam_ppo}. We use Adam optimizer \cite{kingma2014Adam} with a learning rate also specified by the table. 

When in a continuous state space, we implement DDQN \cite{Hasselt2016DDQN} (rather than DQN) with a target network that gets updated at some frequency specified by Table \ref{tab:hyperparam_ppo}.

\begin{table*}[!ht]
\caption{Hyperparameters for Q-learning experiments (Discrete Action Space)}
\label{tab:hyperparam_q}
\begin{center}
\begin{small}
\begin{tabular}{llllllllllll}
\toprule 
 {} & \multicolumn{2}{c}{$\eta$} &   \multicolumn{3}{c}{$\eta$ Decay}  \\ 
 \cmidrule(lr){2-3} \cmidrule(lr){4-6}
 Experiment & $\eta_0$ & Min $\eta$ & Type & Rate & Freq & Batch size &  $K$ ($\#$ batches) & LR & Target update & T & $\gamma$\\
  \midrule 
Minecraft & .3 & 0 & Exponential & .9 & 100 & 128 & 20 & - & - & 100 & .99\\
Pacman & .4 & 0 & Linear & .05 & 400 & 512 & 200 & - & - & 100 & .999\\
Flatworld 1& .8 & .15 & Exponential & .9 & 100 & 128 & 5 & .001 & 15 & 20 & .95\\
Flatworld 2& .8 & .15 & Exponential & .9& 100 & 128 & 5 & .001 & 15 & 50 & .95\\
\end{tabular}
\end{small}
\end{center}
 \end{table*}

\textbf{PPO experiments.} Let $k$ be the greatest number of jump transitions available in some LDBA state $k = \max_{b \in \SState^{\B}} |\A^\B(b)|$. The neural network $f_\theta(s)$ takes as input $s \in \SState^{\mathcal{M}}$ and outputs $\R^{(k+2) \times |\SState^{\B}|}$ is a $(k+2)$-dim vector for each $b \in \SState^{\B}$. For our purposes, we consider $f_\theta(s, b)$ to be the single $(k+2)$-dim vector cooresponding to the particular current state of the LDBA $b$.

$f_\theta(s,b)$ is parameterized by $3$ linear layers with hidden dimension $64$ with intermediary ReLU activations.  The first dimension corresponds to sampling a Gaussian action $a \sim \mathcal{N}(f_\theta(s,b)[0], \text{diag}(\sigma^2))$ where $\sigma$ is initialized to $\sigma_0$ (see Table \ref{tab:hyperparam_ppo}) and decays exponentially (at a rate given in the table) every $10$ episodes. The remaining $k+1$ dimensions (after proper masking to account for the size of $|\A^\B(b)|$ and softmax) represent the probability $p = [p_{a},p_{\epsilon_0}, \ldots, p_{\epsilon_k}]$ of taking either the MDP action $a$ or a some jump transition $\epsilon_i$. We sample from a Categorical($p$) variable to select whether to return $a \sim \mathcal{N}(\text{Tanh}(f_\theta(s,b)[0]), \text{diag}(\sigma^2))$ or $a = \epsilon_i$ for some $i$. The density can be calculated by multiplying $p_a$ by the Gaussian density when $a$ is selected, and $p_{\epsilon_i}$ otherwise.

For the critic, we have a parametrized network $f_\phi(s,b) \to \mathbb{R}$ of $3$ linear layers with hidden dimension $64$ with intermediary Tanh activations and no final activation. 

At each episode (after a rollout of length $T$), we perform $5$ gradient steps with different batches of size given in Table \ref{tab:hyperparam_ppo}. The importance sampling term in PPO is clipped to $1 \pm .4$. The critic learning rate is $.01$. We use Adam optimizer \cite{kingma2014Adam} for both the actor and critic.

\begin{table*}[!ht]
\caption{Hyperparameters for PPO experiments (Continuous Action Space)}
\label{tab:hyperparam_ppo}
\begin{center}
\begin{small}
\begin{tabular}{lllllll}
\toprule
 Experiment & $\sigma_0$  & $\sigma$ Decay Rate & Min $\sigma$ & Batch size & LR Actor & T\\
 \midrule
Flatworld 3 & $1.8$ & $.98$ & $.3$ & $128$ & $.001$ & 20\\
Flatworld 4 & $1.8$ & $.99$ & $.1$ & $128$ & $.001$ & 50\\
Carlo & $.5$ & $.999$ & $.3$ & $16$ & $.0001$ & 500\\
\end{tabular}
\end{small}
\end{center}
 \end{table*}

\newpage
\section{Constructing feasible trajectories for policy gradient during rollout}\label{app:LCER-pg}

Suppose we wanted to generate feasible trajectories in realtime while the policy is being rolled out. That is, we have a partial trajectory of the form $\tau_t = (s_0, b_0, a_0,\ldots, s_t, b_t)$ generated by running $\pi$ in $P$. Let $a_t = a \in \A$ be the $t$-th action taken by $\pi$ and $s_{t+1} = s' \in \mathcal{M}$ be the next observed state observed in the MDP.

Let $\mathcal{T}_t$ be the current set of feasible (partial) trajectories at timestep $t$. Elements $\tau_k = (s_0, b_0, a_0,\ldots, s_k, b_k) \in \mathcal{T}_t$ denote $k$-step (partial) trajectory, not necessarily part of the trajectory observed during the course of a rollout of $\pi$. Here, $k \geq t$. 
Then, for each $\tau_k \in \mathcal{T}_t$, one of 4 cases holds:

\textbf{Case 1.} Action $a$ is not a jump transition (ie. $a \in \A^\M(s_k)$) and there are no jump transitions available in $b_k$ ($\A^\B(b_k) = \varnothing$). Then we can form the concatenation: $\tau_{k+1} = \tau_k \cup (a, s', b_{k+1})$ where $b_{k+1} = P^\B(b_k, L^\mathcal{M}(s'))$. We set $\mathcal{T}_{\epsilon} = \varnothing$.

\textbf{Case 2.} Action $a$ is a jump transition and is currently feasible in $b_k$ (ie. $a \in \A^\M(b_k)$). Then we can form the concatenation $\tau_{k+1} = \tau_k \cup (a, s', b_{k+1})$ where $b_{k+1} = P^\B(b_k, a)$. We set $\mathcal{T}_{\epsilon} = \varnothing$.


\textbf{Case 3.} Action $a$  is a not a jump transition (ie. $a \in \A^\M(s_k)$), but there is at least one feasible jump transition in $b_k$ (ie. $\A^\B(b_k) \neq \varnothing$). Then, in addition to forming $\tau_{k+1}$ from Case 1, we have all the possible jumps:
\begin{equation*}
    \mathcal{T}_{\epsilon} = \{\tau_k \cup (\epsilon, s_k, b_{k+1}, a, s', b_{k+2}) | \forall \epsilon \in \A^\B(b_k),\\ b_{k+1} = P^\B(b_k, \epsilon), b_{k+2} = P^\B(b_{k+1}, a_t)\}
\end{equation*}

\textbf{Case 4.} Action $a$  is a jump transition is infeasible in $b_k$ (ie. $ a \not\in \A^\B(b_k)$). In this case, we just pass this trajectory. Setting $\tau_{k+1} = \tau_k$ and $\mathcal{T}_\epsilon = \varnothing$.

At the end of iterating over each element of $\tau_k \in \mathcal{T}_t$ and forming $\tau_{k+1}$ and $\mathcal{T}_{\epsilon}$, we can update our current set of feasible trajectories:
\begin{equation}\label{eq:pg-update}
    \mathcal{T}_{t+1} = \cup_{\tau_k \in \mathcal{T}_{t}} \bigg((\mathcal{T}_{t} \setminus \{\tau_k\})\cup \{\tau_{k+1}\} \cup \mathcal{T}_{\epsilon}\bigg)
\end{equation}

To put this process simply, we are swapping out $\tau_k$ for $\tau_{k+1}$ and also adding in any jump transitions if they are available. The algorithm can be seen in Algo \ref{algo:pg_2}.

\begin{algorithm}[!ht]
	\caption{$\texttt{LCER}$ for Policy Gradient (Option 2)} 
	\label{algo:pg_2}
	\begin{algorithmic}[1]
            \REQUIRE Dataset $D$. Trajectory $\tau$ of length $T$.
            \STATE Set $\mathcal{T}_0 \leftarrow \{(s_0, b) | b \in \B\}$
            \FOR{$(s_t, a_t, s_{t+1}) \in \tau$}
                \STATE Form  $\mathcal{T}_t$ according to Eq \eqref{eq:pg-update}
            \ENDFOR
            \STATE Set $D \leftarrow D \cup \mathcal{T}_{T}$
            \RETURN{$D$}
	\end{algorithmic}
\end{algorithm}




\end{document}